%% file: main.tex
\documentclass{article}

\usepackage{microtype}
\usepackage{graphicx}
\usepackage{subcaption}
\usepackage{booktabs} %

\usepackage{hyperref}

\usepackage[preprint]{icml2026}

\usepackage{amsmath}
\usepackage{amssymb}
\usepackage{mathtools}
\usepackage{amsthm}

\usepackage[capitalize,noabbrev]{cleveref}

\theoremstyle{plain}
\newtheorem{theorem}{Theorem}[section]

\newtheorem{lemma}[theorem]{Lemma}

\theoremstyle{definition}
\newtheorem{definition}[theorem]{Definition}

\theoremstyle{remark}

\usepackage[textsize=tiny]{todonotes}

\icmltitlerunning{Knowledge Gradient for Preference Learning}

\usepackage{macro}
\usepackage{lipsum}

\begin{document}

\twocolumn[
  \icmltitle{Knowledge Gradient for Preference Learning}

  \icmlsetsymbol{equal}{*}

  \begin{icmlauthorlist}
    \icmlauthor{Kaiwen Wu}{upenn}
    \icmlauthor{Jacob R. Gardner}{upenn}
  \end{icmlauthorlist}

  \icmlaffiliation{upenn}{University of Pennsylvania, Philadelphia, PA, United States}

  \icmlcorrespondingauthor{Kaiwen Wu}{kaiwenwu@seas.upenn.edu}
  \icmlcorrespondingauthor{Jacob R. Gardner}{jacobrg@seas.upenn.edu}

  \icmlkeywords{Machine Learning}

  \vskip 0.3in
]

\printAffiliationsAndNotice{}  %

\begin{abstract}
The knowledge gradient is a popular acquisition function in Bayesian optimization (BO) for optimizing black-box objectives with noisy function evaluations.
Many practical settings, however, allow only pairwise comparison queries, yielding a preferential BO problem where direct function evaluations are unavailable.
Extending the knowledge gradient to preferential BO is hindered by its computational challenge.
At its core, the look-ahead step in the preferential setting requires computing a non-Gaussian posterior, which was previously considered intractable.
In this paper, we address this challenge by deriving an exact and analytical knowledge gradient for preferential BO.
We show that the exact knowledge gradient performs strongly on a suite of benchmark problems, often outperforming existing acquisition functions.
In addition, we also present a case study illustrating the limitation of the knowledge gradient in certain scenarios.
\end{abstract}

\section{Introduction}
\label{sec:intro}
\input{paper/introduction}

\section{Background}
\label{sec:background}
\input{paper/background}

\section{Analytical Knowledge Gradient for Preferential Bayesian Optimization}
\label{sec:method}
\input{paper/method}

\section{Experiments}
\input{paper/experiments}

\section{Related Work}
\input{paper/related}

\section{Discussion}
\input{paper/discussion}

\clearpage

\section*{Impact Statement}
This paper presents work whose goal is to advance the field of machine learning.
There might be many potential societal consequences of our work, none of which we feel must be specifically highlighted here.

\bibliography{ref}
\bibliographystyle{icml2026}

\newpage
\appendix
\onecolumn

\section{Extended Skew Normal Distributions}
\label{sec:extended-skew-normal}
\input{appendix/extended-skew-normal}

\section{Proofs}
\label{sec:proofs}
\input{appendix/proofs}

\section{Additional Experimental Details}
\label{sec:additional-experiments}
\input{appendix/additional-experiments}

\end{document}

%% file: paper/introduction.tex
Preference feedback is often the most practical or sometimes even the only available supervision signal in many real-world applications.
As a result, preference learning is prominent in various fields including A/B testing,
recommender systems \citep{furnkranz2010preference},
human-in-the-loop psychophysics experimental design \citep{letham2022look,wu2025mixed},
robotics \citep{woodworth2018preference,li2021roial},
reinforcement learning \citep{christiano2017deep},
and more recently large language model alignment \citep{ziegler2019fine,stiennon2020learning,ouyang2022training,rafailov2023direct}.

This paper is concerned with adaptively learning preferences from pairwise comparisons.
Namely, we wish to learn the preferences by iteratively collecting data from an oracle (\eg, the user) and eventually identify the most preferred item.
Preferential Bayesian optimization (PBO) is a popular framework for this problem \citep{gonzalez2017preferential}.
In PBO, the preferences are generated according to a hidden latent utility function, which is learned by a probabilistic model.
Then, the probabilistic model is leveraged to guide the data collection process by selecting the most informative pairwise comparisons via an acquisition function.

The knowledge gradient is a popular acquisition function for (standard) BO that leverages look-ahead posterior updates to make informed decisions \citep{frazier2018tutorial}.
The main idea is to select the query candidate that maximizes the expected increase in the maximum posterior mean of the surrogate model.
The knowledge gradient is known to be one-step Bayes optimal, and it has been shown to perform well empirically \citep{frazier2018tutorial}.
Initially proposed for the ranking and selection problem over discrete domains \citep{frazier2008knowledge,frazier2009knowledge}, the knowledge gradient has been extended to various settings.
These extensions include continuous domains \citep{scott2011correlated}, batch BO \citep{wu2016parallel}, and problems with derivative observations \citep{wu2017bayesian}.
On the computational side, \citet{balandat2020botorch} have proposed computing the knowledge gradient via sample average approximation and a one-shot formulation, sidestepping the nested expectation and maximization.
More recently, \citet{daulton2023hypervolume} have developed the hypervolume knowledge gradient, a generalization to multi-objective BO.

However, instantiating the knowledge gradient for PBO has remained elusive due to its computational intractability.
The core difficulty is the look-ahead step, where preference observations induce a non-Gaussian posterior that breaks GP conjugacy.
To the best of our knowledge, the only two previous attempts in the literature applying the knowledge gradient to PBO are \citet{lin2022preference,astudillo2023qeubo}.
However, both of them resort to approximation that removes the challenging non-Gaussian look-ahead step.

In this paper, we show that the knowledge gradient acquisition function for pairwise preference learning surprisingly admits a closed-form expression.
This result only requires two mild assumptions: (a) the underlying surrogate model is a GP and (b) the data generation process follows a probit model.
Both are common choices ubiquitous in PBO.
Our key insight is that the look-ahead posterior follows an extended skew normal distribution, whose expectation is already available analytically \citep[\eg,][]{azzalini2013skew}.
As a result, we obtain an analytic formula for the look-ahead posterior mean despite that it is non-Gaussian.
Next, plugging it into the one-shot knowledge gradient formulation \citep{balandat2020botorch} yields an exact knowledge gradient acquisition function for PBO that requires no approximation.

Our results extend the applicability of the knowledge gradient to PBO, providing a principled and efficient alternative to existing methods.
Empirically, we demonstrate the effectiveness of the proposed knowledge gradient acquisition function on a suite of benchmark problems commonly used in the literature, where it is competitive against or outperform existing acquisition functions for PBO.
Through case studies, we also illustrate the potential failure mode of the knowledge gradient acquisition function in certain scenarios and contrast the behavior of the exact knowledge gradient with the existing approximate knowledge gradient by \citet{lin2022preference,astudillo2023qeubo}.

%% file: paper/background.tex
Let $f \sim \mathcal{GP}(\mu, k)$ be a Gaussian process (GP) defined on a domain $\Xc$ with a mean function $\mu: \Xc \to \Rb$ and a kernel function $k: \Xc^2 \to \Rb$.
The function values of $f$ at any finite collection of inputs follow a multivariate normal distribution determined by the mean and kernel functions of the GP.
The posterior of $f$ conditioning on a finite set of data is also a GP with updated mean and covariance functions \citep{williams2006gaussian}, both available in closed forms.

Bayesian optimization (BO) is a framework for optimizing black-box expensive-to-evaluate functions \citep{garnett2023bayesian}.
BO is concerned with maximizing a black-box function $f$ over a domain $\Xc$ with noisy function evaluations:
\begin{equation}
\label{eq:maximizing-black-box}
    \maxi_{\xv \in \Xc} f(\xv).
\end{equation}
The noisy function evaluations are given by $y = f(\xv) + \epsilon$ with $\epsilon$ being a Gaussian noise.
Let $\Dc = \cbb{(\xv_i, y_i)}_{i=1}^n$ be the set of such observations.
BO proceeds by constructing a probabilistic surrogate model based on the observations $\Dc$, where the surrogate model is often a GP model.
If we impose a GP prior on the unknown function $f$, then the posterior GP $f \mid \Dc$ gives the prediction for the unknown function along with with predictive uncertainty.

Then, BO iteratively selects next query candidates in the domain by maximizing an acquisition function.
The acquisition function measures the informativeness of each query candidate (the higher the value, the more informative) and is constructed based on the belief of the model about the unknown function.

\subsection{Knowledge Gradient}
The particular acquisition function that we will be focusing on in this paper is the knowledge gradient, a widely used look-ahead acquisition function in BO \citep{frazier2008knowledge,frazier2009knowledge,wu2016parallel,frazier2018tutorial}.
The main idea is using a look-ahead step to monitor how much the maximum GP posterior mean (\ie, the knowledge) would increase if a particular query candidate was selected.
The query outcome is unknown at the time of query selection, and thus the knowledge gradient takes an expectation over all possible fantasized outcomes according to the model belief.
In standard BO with noisy function evaluations, the knowledge gradient $\operatorname{kg}(\xv)$ of a particular query candidate $\xv$ is given by
\[
\Eb_{
    p\bb{y \mid \Dc}
} \\
\sbb[\bigg]{
    \max_{\xv^\prime \in \Xc}
    \Eb\sbb{
        f(\xv^\prime) \mid \Dc, \bb{\xv, y}
    }
}
-
\max_{\xv^\prime \in \Xc} \Eb\sbb{
    f(\xv^\prime) \mid \Dc
},
\]
where $y$ is the fantasized outcome and $p(y \mid \Dc)$ is the predictive distribution of the outcome by the GP $f$ given the current data $\Dc$.
The inner expectation is taken over the randomness of the probabilistic model $f$, while the outer expectation is taken over the predictive distribution of the fantasized outcome $y$.
Maximizing the knowledge gradient over $\xv$ produces the next query candidate to evaluate.
Note that the knowledge gradient involves a nested maximization, where $\xv^\prime$ is maxed out in the inner loop.

Due to the nested maximization and expectation, computing the knowledge gradient acquisition function is challenging in general.
Thus, one has to resort to approximation, \eg, discretizing the domain \citep{frazier2009knowledge} or using Monte Carlo methods to approximate the expectation and its gradient \citep[\eg,][]{wu2016parallel,wu2017bayesian,balandat2020botorch}.

\subsection{Preferential Bayesian Optimization}
Preferential Bayesian optimization (PBO) is concerned with maximizing the same black-box function $f$ as in \eqref{eq:maximizing-black-box}.
But this time the oracle is restricted to preferential observations.
The function $f$ is viewed as a latent utility function that governs the preference (the higher the utility, the more preferred).
In particular, we will be focusing on pairwise preferential comparisons:
$\xv_1 \succ \xv_2$ denotes $\xv_1$ is preferred to $\xv_2$,
and $\xv_1 \prec \xv_2$ denotes $\xv_2$ is preferred to $\xv_1$.
The latent function $f$ is not directly observed anymore.

Each pairwise comparison outcome is assumed to follow the preferential probit likelihood \citep{chu2005preference}:
\begin{equation}
\label{eq:preferential-probit-likelihood}
    \Pr(\xv_1 \succ \xv_2) = \Phi\bb[\big]{\tfrac{1}{\sigma} \bb{f(\xv_1) - f(\xv_2)} \mid f},
\end{equation}
where $\Phi(\,\cdot\,)$ is the standard normal CDF.
Here, $\sigma > 0$ is a scalar parameter that determines the noise in the pairwise comparison.
The outcome is deterministic as $\sigma \to 0$, and becomes fully random as $\sigma \to \infty$.
Note that the noise parameter $\sigma$ is not identifiable, since it can be absorbed into the scale of $f$.
Thus, there is some flexibility in choosing $\sigma$ when learning the latent function.

Let $\Dc$ be a set of preferential observations.
If we impose a GP prior on the latent utility $f$, then the posterior $p(f \mid \Dc)$ is non-Gaussian due to the preferential likelihood \eqref{eq:preferential-probit-likelihood}, which makes the downstream tasks challenging (\eg, acquisition function computation).
Thus, it is common to approximate this posterior by a GP using approximate inference such as Laplace approximation or variational inference \citep{chu2005preference,hensman2015scalable}.
From now on, we will assume that such an approximation is employed (unless stated otherwise), and we (somewhat confusingly) interpret $f \mid \Dc$ as the approximate GP posterior.

The knowledge gradient imposes a greater computational challenge in preferential BO as the pairwise comparison outcome follows a non-Gaussian distribution.
Even with the approximate GP posterior (which conditions on $\Dc$), the look-ahead posterior (which conditions on the duel between $\xv_1$ and $\xv_2$) in the knowledge gradient is still (seemingly) intractable.
In particular, previous attempts in the literature sidestep this computational challenge in the look-ahead step by approximation.
For example, \citet{lin2022preference} propose the expected utility of the best option (EUBO) as a proxy to the knowledge gradient:
\[
    \operatorname{EUBO}\bb{\xv_1, \xv_2} = \Eb \sbb[\big]{\max\cbb{f(\xv_1), f(\xv_2)} \mid \Dc},
\]
where the expectation is taken over the joint (approximate) posterior of $f(\xv_{1,2})$ conditioned on the data $\Dc$.
EUBO conveniently has an analytical expression if the posterior $f \mid \Dc$ is approximated by a GP, which is almost always the case in PBO.
To justify this approximation, \citet{lin2022preference} show that the maximizers of EUBO also maximize the exact knowledge gradient (which we will define in \S\ref{sec:method}) if the noise parameter $\sigma$ in the preferential probit likelihood \eqref{eq:preferential-probit-likelihood} vanishes.\footnote{
This original result \citet{lin2022preference} is proved for the logistic likelihood.
But the same proof applies to the probit likelihood as well.
}
Namely, EUBO can be viewed as an approximate knowledge gradient with a noiseless look-ahead step.
With a non-zero noise parameter, however, EUBO is only a lower bound of the exact knowledge gradient.

Subsequently, \citet{astudillo2023qeubo} generalizes EUBO to comparisons over a batch of $q$ candidates:
\[
    \operatorname{qEUBO}\bb{\xv_{1:q}} = \Eb \sbb[\bigg]{\max_{1 \leq i \leq q} f(\xv_i) \mid \Dc},
\]
where again the expectation is taken over the joint posterior of $f$ conditioned on the data $\Dc$.
Unfortunately, qEUBO does not have an analytical expression anymore, and thus has to relies on Monte Carlo approximation.

In \S\ref{sec:method}, we will show that the knowledge gradient for PBO with pairwise comparisons can be computed analytically, which renders approximation unnecessary for batch $q = 2$.

%% file: paper/method.tex
\begin{figure*}
\centering
    \includegraphics[width=\linewidth]{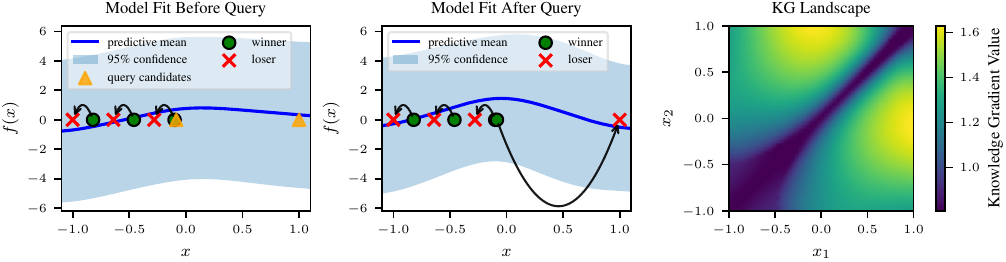}
\caption{
A one-dimensional example of preferential knowledge gradient.
The ground-truth latent function is a quadratic function $f(x) = -x^2$, whose maximum is attained at the origin.
\textbf{Left \& Mid:}
The GP model fit before and after the query.
The blue shaded region indicates two standard deviations of the GP posterior.
The green circles and red crosses are the observed preferential comparisons.
The arrows point from duel winners to duel losers.
The two triangles indicate the query candidates selected by maximizing the knowledge gradient.
\textbf{Right:}
The landscape of knowledge gradient $\operatorname{kg}(x_1, x_2)$ by maxing out the fantasy samples $\xv_{\pm}$.
}
\label{fig:one-dimensional-example}
\end{figure*}

We wish to design a knowledge gradient acquisition function to select a pair of preferential comparison candidates given a GP model $f \sim \mathcal{GP}(\mu, k)$.
Here, we assume $f$ is already the posterior GP conditioned on the existing preferential observations $\Dc$ so that $\mu(\,\cdot\,)$ and $k(\,\cdot\,, \,\cdot\,)$ are the \emph{posterior} mean and covariance functions respectively.
This way we can omit the conditioning on $\Dc$ to simplify the notation.
As discussed in \S\ref{sec:background}, the posterior GP necessarily is obtained by approximate inference since the preferential likelihood \eqref{eq:preferential-probit-likelihood} is non-Gaussian \citep{chu2005preference,hensman2015scalable}.
For simplicity, we fix the noise parameter $\sigma = 1$ in the preferential likelihood \eqref{eq:preferential-probit-likelihood} when computing the look-ahead step.
This is without loss of generality because our results in this section can be easily extended to other values of $\sigma > 0$ by scaling the GP $f$ accordingly.

The pairwise preferential comparison only has two possible outcomes: either \(\xv_1 \succ \xv_2\) or $\xv_1 \prec \xv_2$.
The knowledge gradient in the preferential setting is a weighted average of the two outcomes:\footnote{
Technically, the formula needs to subtract $\max_{\xv \in \Xc}\Eb\sbb{f(\xv)}$.
We omit this term since it is a constant independent of $\xv_{1:2}$.
}
\begin{align}
\label{eq:kg}
\begin{split}
\operatorname{kg} \bb{\xv_{1:2}} & = \Pr\bb{\xv_1 \succ \xv_2 \mid f}
    \max_{\xv \in \Xc} \Eb\sbb{
        f(\xv) \mid \xv_1 \succ \xv_2
    }
\\
+ & \Pr\bb{\xv_1 \prec \xv_2 \mid f}
    \max_{\xv \in \Xc} \Eb\sbb{
        f(\xv) \mid \xv_1 \prec \xv_2
    }.
\end{split}
\end{align}
The two terms $\Pr\bb{\xv_1 \succ \xv_2 \mid f}$ and $\Pr\bb{\xv_1 \prec \xv_2 \mid f}$ have closed forms since integration of the probit function against Gaussian has a closed form.
For example, the first one is computed as
\begin{align*}
\Pr\bb{\xv_1 \succ \xv_2 \mid f}
& =
\int \Phi\bb{f(\xv_1) - f(\xv_2)} \diff f(\xv_1) \diff f(\xv_2) \\
& =
\Phi\bb[\bigg]{
    \frac{
        \Eb\sbb{f(\xv_1) - f(\xv_2)}
    }{
        \sqrt{\operatorname{Var}\sbb{f(\xv_1) - f(\xv_2)} + 1}
    }
},
\end{align*}
where $f(\xv_1)$ and $f(\xv_2)$ are normal random variables jointly following a bivariate normal distribution determined by the GP.
Note that we have fixed the noise parameter $\sigma = 1$ in the preferential probit likelihood \eqref{eq:preferential-probit-likelihood} since this scalar can be pushed into the scale of $f$ when learning the latent function in practice.

To compute the knowledge gradient \eqref{eq:kg}, the one-step look-ahead posterior mean is required:
\begin{equation}
\label{eq:one-step-look-ahead-posterior-mean}
\Eb\sbb{
    f(\xv) \mid \xv_1 \succ \xv_2
},
\end{equation}
where the expectation is taken over the randomness of \(f\).
The challenge here is that the look-ahead posterior is non-Gaussian due to the preferential observation $\xv_1 \succ \xv_2$.
Thus, one would expect to resort to approximation methods such
as Monte Carlo sampling.
However, this expectation \eqref{eq:one-step-look-ahead-posterior-mean} actually has a closed-form despite being non-Gaussian.
\begin{lemma}
\label{thm:closed-form-look-ahead-posterior-mean}
The one-step look-ahead posterior mean \eqref{eq:one-step-look-ahead-posterior-mean} is given by
\begin{equation}
\label{eq:closed-form-one-step-look-ahead-posterior-mean}
\mu(\xv) + \frac{\phi(\tau)}{\Phi(\tau)} \cdot \frac{
    \operatorname{Cov} \sbb{f(\xv), f(\xv_1) - f(\xv_2)}
}{
    \sqrt{\operatorname{Var} \sbb{f(\xv_1) - f(\xv_2)} + 1}
},
\end{equation}
where
\(
\tau = \frac{
    \Eb\sbb{f(\xv_1) - f(\xv_2)}
}{
    \sqrt{\operatorname{Var} \sbb{f(\xv_1) - f(\xv_2)} + 1}
}
\) and $\phi(\,\cdot\,)$ and $\Phi(\,\cdot\,)$ are the standard normal PDF and CDF, respectively.
\end{lemma}
\begin{proof}
A crucial observation is that conditioning on the preferential observation $\xv_1 \succeq \xv_2$ is equivalent to conditioning on an inequality constraint:
\[
\Eb\sbb{f(\xv) \mid \xv_1 \succ \xv_2}
=
\Eb\sbb{f(\xv) \mid f(\xv_1) - f(\xv_2) + \epsilon \geq 0},
\]
where $\epsilon \sim \Nc(0, 1)$ is an independent univariate standard normal random variable.
To see this, recall that the preferential likelihood is given by
\[
\Pr\bb{\xv_1 \succ \xv_2 \mid f} = \Phi(f(\xv_1) - f(\xv_2)).
\]
The value of the right-hand side happens to be exactly the same as the probability of the linear inequality:
\[
    \Pr\bb{f(\xv_1) - f(\xv_2) + \epsilon \geq 0 \mid f}.
\]
Thus, the inequality constraint describes exactly the same event as the observation $\xv_1 \succ \xv_2$.
Note that $f(\xv)$ and $f(\xv_1) - f(\xv_2) + \epsilon$ follow a bivariate normal distribution if $f$ is a GP.
Then, by \Cref{thm:conditional-mean-bivariate-normal}, the conditional distribution
\[
    f(\xv) \mid f(\xv_1) - f(\xv_2) + \epsilon \geq 0
\]
follows an extended skew normal distribution, whose mean is exactly given by \eqref{eq:closed-form-one-step-look-ahead-posterior-mean}.
\end{proof}
Compared to $\mu(\xv)$, the mean before the look-ahead step, whether the posterior mean after the look-ahead step increases or not entirely depends on the correlation between $f(\xv)$ and $f(\xv_1) - f(\xv_2)$.
If they are positively correlated, the look-ahead posterior mean will increase with the observation $\xv_1 \succ \xv_2$.
Intuitively, this is because observing $\xv_1 \succ \xv_2$ suggests that $f(\xv_1)$ is likely larger than $f(\xv_2)$, which in turn suggests that look-ahead posterior mean will likely become larger than due to the positive correlation.
The magnitude of increase is scaled by the uncertainty in the latent utility difference $f(\xv_1) - f(\xv_2)$, as well as the PDF and CDF terms.

Note that \eqref{eq:closed-form-one-step-look-ahead-posterior-mean} only depends on the GP mean, variance and covariance, which are all available in closed forms.
Thus, the one-step look-ahead posterior mean conditioning on the outcome \(\xv_1 \succ \xv_2 \) can be computed analytically.
The other case \(\xv_1 \prec \xv_2\) clearly also has a closed form available thanks to symmetry.

Remarkably, now we can compute knowledge gradient for preference learning in a closed form.
To appreciate this, recall that the knowledge gradient acquisition function does not have a closed form even for the Gaussian likelihood widely used in the standard BO.
In fact, there has been a line of work dedicated to approximating the knowledge gradient acquisition function in these settings \citep[\eg,][]{wu2016parallel,wu2017bayesian,balandat2020botorch}.
Here, the non-Gaussianity in the preferential learning surprisingly makes the knowledge gradient computation easier.

The proof of \Cref{thm:closed-form-look-ahead-posterior-mean} is based on a simple observation that the one-step look-ahead posterior is an extended skew normal distribution, whose moments are readily available in the existing literature \citep[\eg,][]{azzalini2013skew}.
We direct readers to \S\ref{sec:extended-skew-normal} for more background on extended skew normal distributions and \S\ref{sec:proofs} for the detailed proof.

\subsection{One-Shot Knowledge Gradient}
Note that the knowledge gradient acquisition function \eqref{eq:kg} itself is defined by a maximization.
To generate candidates, it is required to maximize the knowledge gradient \eqref{eq:kg}, which yields a max-max problem.
We squash the nested maximization and optimize the ``augmented'' version by introducing two auxiliary variables $\xv_-$ and $\xv_+$ as follows:
\[
\maxi_{\xv_{1:2}, \xv_{\pm} \in \Xc} \operatorname{kg}(\xv_1, \xv_2, \xv_+, \xv_-),
\]
where
\begin{align*}
\operatorname{kg}(\xv_{1:2}, \xv_{\pm})
& =
\Pr\bb{\xv_1 \succ \xv_2 \mid f} \cdot
\Eb\sbb{
    f(\xv_+) \mid \xv_1 \succ \xv_2
}
\\
+ &
\Pr\bb{\xv_1 \prec \xv_2 \mid f} \cdot
\Eb\sbb{
    f(\xv_-) \mid \xv_1 \prec \xv_2
}.
\end{align*}
Once finishing optimizing the augmented acquisition function, we discard $\xv_{\pm}$ and only keep the query candidates $\xv_{1:2}$.
This trick is commonly known as the one-shot knowledge gradient, which squashes the nested maximization into a single maximization at the cost of raising the dimensionality of the problem \citep[\eg,][]{balandat2020botorch}.
The auxiliary variables $\xv_{\pm}$ are usually called fantasy samples.

\Cref{fig:one-dimensional-example} illustrates a one-dimensional example of the preferential knowledge gradient.
The knowledge gradient acquisition function is symmetric with respect to its two inputs $\xv_{1:2}$.
It usually has low values along the diagonal region where $\xv_1$ and $\xv_2$ are close to each other since such queries are less informative.
Before the query, the GP model has all samples on the left side ($x \leq 0$), and thus is uncertain about the latent function on the right side ($x > 0$).
Accordingly, the knowledge gradient acquisition function chooses to put one of the query candidates to the right side.
After the query, the GP model curves upwards around the origin where the ground-truth maximum is located.
Note that the uncertainty of the learned GP is quite large

\subsection{Exact Look-Ahead \vs Approximate Model}
Our knowledge gradient formulation is exact by leveraging an exact look-ahead step.
The GP model that the knowledge gradient formulation relies on, however, is still approximate as we assume the exact posterior $f \mid \Dc$ is approximated by a GP.
In other words, we have developed an exact knowledge gradient for approximate models in preferential BO.
Thus, the look-ahead posterior mean in \eqref{eq:closed-form-one-step-look-ahead-posterior-mean} is almost always not the same as the posterior mean obtained by approximate inference after conditioning on the new observation.
However, the look-ahead posterior mean still captures the essential information of how the posterior mean would likely to change, and thus provides learning signals for selecting informative queries.
This is similar in spirit to \citet{letham2022look}, who develop an exact look-ahead step for approximate models in the context of Bernoulli level set estimation.

%% file: paper/experiments.tex
\begin{figure*}
    \includegraphics[width=\textwidth]{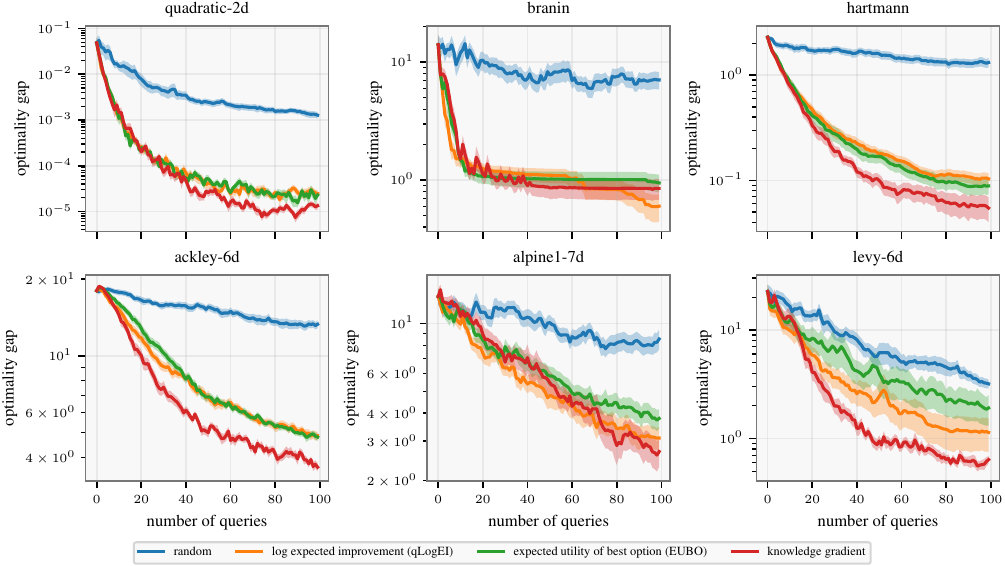}
\caption{
\textbf{(Low Noise)}
The optimality gaps in log scales against the number of queries for random sampling, log expected improvement, EUBO, and the knowledge gradient acquisition function.
A small amount of noise is added to the pairwise comparisons to simulate noisy preferences.
The knowledge gradient acquisition function outperforms all baselines on all problems except one.
Note that the Ackley function here uses the original domain $[-32.768, 32.768]^d$.
}
\label{fig:benchmark-noisy-low}
\end{figure*}

In this section, we empirically evaluate the exact knowledge gradient acquisition function in \S\ref{sec:method} for PBO.
We present results on $6$ benchmark functions against various baselines ranging from 2 dimensions to 7 dimensions.
The details for the benchmark test functions are deferred to \S\ref{sec:additional-experiments}.

The surrogate model for the latent utility function $f$ is a variational GP \citep{hensman2015scalable}.
We use a constant zero function as the GP prior mean since the preferential probit likelihood \eqref{eq:preferential-probit-likelihood} is translation invariant.
The GP uses the Mat\'ern kernel with $\nu = 2.5$ and ARD.
We impose Gamma priors on the kernel output scale and length scales.
The GP is trained by variational inference using the preferential probit likelihood \eqref{eq:preferential-probit-likelihood} with $\sigma = 1$.
Note that the noise parameter $\sigma$ here might be different from the one used to generate the preferential outcomes.
That is, the GP model aims to learn a scaled version of the ground-truth latent utility function.
All variaitonal GPs use at most $200$ inducing points that are selected from two sources:
(a) $100$ Sobol samples in the domain $\Xc$ and (b) the training inputs in the existing observations $\Dc$.
If the two sources yield more than $200$ points, we select the top $200$ points from them by greedy variance reduction, an inducing point allocation strategy by \citet{burt2020convergence}.
We use the off-the-shelf implementation with the same name in BoTorch \citep{balandat2020botorch}.

For a $d$-dimensional input space, all methods are initialized with $4d$ pairs of queries selected by Sobol sampling.
We run all methods for $100$ iterations after the initialization which yields a total of $4d + 100$ preferential outcomes.
The exact knowledge gradient uses a fixed noise parameter $\sigma = 1$ in the look-ahead step as developed in \S\ref{sec:method} and is compared against the following baselines.

\paragraph{Random}
This method selects a pair of query candidates by Sobol sampling in each iteration.
The GP model is still employed to estimate the maximizer of the latent utility function, but is not used to select queries.

\paragraph{LogEI}
An acquisition function designed for standard BO with Gaussian likelihoods \citep{ament2023unexpected}.
We use its batch version to select two query candidates at each iteration.
Note that this acquisition function is not modified in any way to handle preferential observations; we use it as it is.
The best incumbent value is set to the maximum posterior mean of the GP model on the training inputs.
This is consistent with \citet{lin2022preference,astudillo2023qeubo}.

\paragraph{EUBO} The expected utility of the best option (EUBO) by \citet{lin2022preference}, which can be viewed as an approximation to the knowledge gradient for PBO.
We use its implementation in BoTorch \citep{balandat2020botorch}.

\subsection{Results}
We present benchmark results in \Cref{fig:benchmark-noisy-low,fig:benchmark-noisy-high}, where we report the optimality gap against the number of queries.
The optimality gap at iteration $i$ is defined as
\[
    f_\mathrm{true}^* - f_\mathrm{true}(\hat{\xv}_i)
\]
where $f_\mathrm{true}^*$ is the maximum of the ground-truth latent utility function and
\[
    \hat{\xv}_i = \argmax_{\xv \in \Xc} \Eb\sbb{f(\xv) \mid \Dc_i},
\]
where $\Dc_i$ is the observations collected until the $i$-th iteration and $\hat\xv_i$ is the estimated maximizer in the $i$-th iteration based on the GP model $f$.

We add noises to pairwise comparisons to simulate inconsistent preferences by a positive noise parameter in the preferential likelihood \eqref{eq:preferential-probit-likelihood}
This parameter is tuned to each test function.
In \Cref{fig:benchmark-noisy-low}, the noise parameter $\sigma$ is tuned such that the pairwise comparisons among the top $1\%$ of the domain have errors with a probability of $10\%$.
In \Cref{fig:benchmark-noisy-high}, the noise parameter is tuned so that the error probability among the top $1\%$ of the domain increases to $30\%$.
The experiments are averaged over $20$ random seeds.
Note that the Ackley function in the experiments uses the original domain, while the corresponding experiment in \citet{astudillo2023qeubo} uses a smaller domain $[-2, 2]^d$.\footnote{See \url{https://tinyurl.com/2n395t7y}}

The knowledge gradient acquisition function is competitive across the benchmarks in both low-noise and high-noise settings, outperforming nearly all baselines in the first $100$ iterations.
We remark that the log expected improvement is competitive against EUBO and the knowledge gradient, even though it is not designed for preferential BO.

\begin{figure*}
    \includegraphics[width=\textwidth]{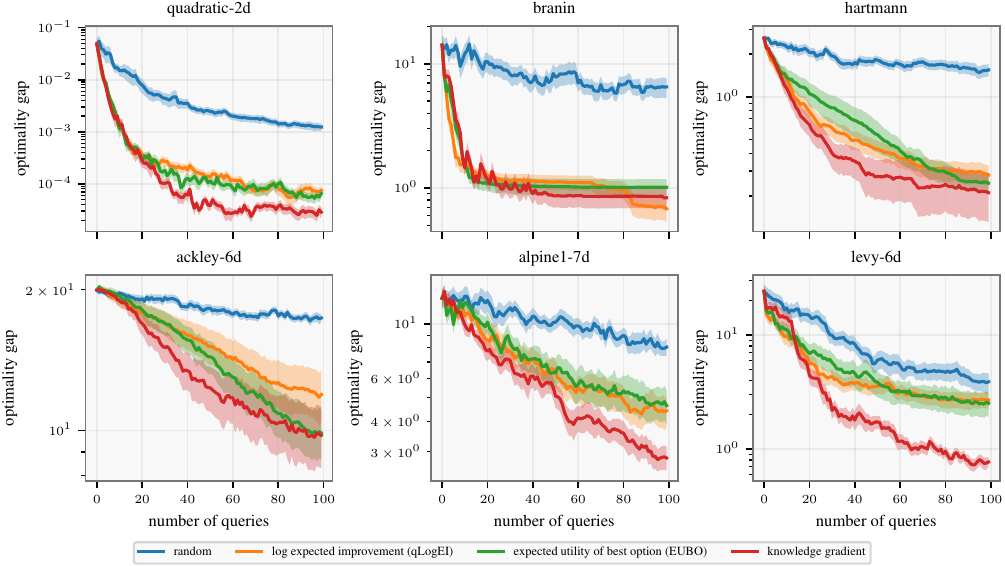}
\caption{
\textbf{(High Noise)}
The optimality gaps in log scales against the number of queries for random sampling, log expected improvement, EUBO, and
the knowledge gradient acquisition function.
The setting is the same as in \Cref{fig:benchmark-noisy-low} except that the pairwise comparisons are noisier.
}
\label{fig:benchmark-noisy-high}
\end{figure*}

\subsection{Case Study: 2D Levy function}
In this section, we present a case study on the $2$-dimensional Levy function.
In \Cref{fig:case-study-levy}, we visualize the queries selected by EUBO and the knowledge gradient.
The queries selected by the two methods exhibit very different patterns after they identify the promising region around the global maximum.
In particular, EUBO tends to select queries that ``collapse'' around the estimated maximum, whereas the knowledge gradient is less aggressive and selects the queries centering around the estimated maximizer.

This case study suggests that EUBO put more emphasis on exploitation compared to the knowledge gradient, which is consistent with the fact that EUBO is roughly equivalent to the knowledge gradient with a deterministic look-ahead step with the noise parameter $\sigma \to 0$ in the preferential probit likelihood \eqref{eq:preferential-probit-likelihood} \citep{astudillo2023qeubo}.
This exploitation behavior also checks out with the expression of the EUBO acquisition function:
\[
    \Eb\sbb[\big]{\max \cbb{f(\xv_1), f(\xv_2)} \mid \Dc},
\]
where the acquisition function value is likely higher if both $\xv_1$ and $\xv_2$ are close to the estimated maximizer.

In contrast, our exact knowledge gradient is computed with a noisy look-ahead step, where the look-ahead posterior is computed with the preferential probit likelihood \eqref{eq:preferential-probit-likelihood} with a unit noise parameter $\sigma = 1$.
Thus, the knowledge gradient acquisition function is reluctant to select two queries that are too close to each other, in which case the look-ahead step thinks the outcome would be noisy and not as informative.
Instead, the knowledge gradient selects queries that center around the estimated maximizer trying to learn the slope of the latent utility function and extrapolate the location of the global maximum.

This case study suggests that the approximate knowledge gradient, \eg, EUBO, can behave quite differently from the exact knowledge gradient.
In this case study, EUBO does achieve a smaller optimality gap than the exact knowledge gradient, showing that the exploitation behavior could be beneficial in certain cases, especially when the noise is small and the promising region predicted by the model contains the global optimizer.
The problem with the exact knowledge gradient in this case is that the look-ahead step uses a noise parameter that is larger than it supposed to be, and thus the exact knowledge gradient behaves more conservative than it should be.

\begin{figure*}
    \includegraphics[width=\textwidth]{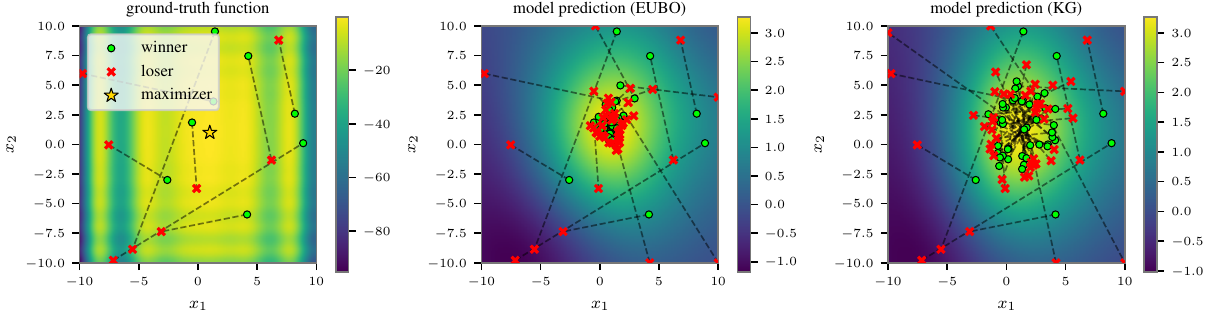}
\caption{
A comparison of the queries selected by EUBO and the knowledge gradient on the 2D Levy function.
The preferential outcomes are generated deterministically in this case study, \ie, the noise parameter $\sigma$ in the preferetial probit likelihood \eqref{eq:preferential-probit-likelihood} vanishes.
\textbf{Left:}
The queries selected by Sobol during initialization and the ground-truth latent utility function.
\textbf{Mid:}
The queries selected by EUBO after $50$ queries and the GP posterior mean.
\textbf{Right:}
The queries selected by the knowledge gradient after $50$ queries and the GP posterior mean.
Even though both methods identify the region of the optimizer, the queries have very different patterns.
In particular, EUBO tends to ``collapse'' the two queries close to the estimated maximum, while the knowledge gradient selects queries centers around the estimated maximum.
In this case, the optimality gap of EUBO is $0.0182$ while that of the knowledge gradient is $0.0191$ and EUBO achieves a better performance.
}
\label{fig:case-study-levy}
\end{figure*}

%% file: paper/related.tex
In this section, we discuss additional related work complementing \S\ref{sec:intro} and \S\ref{sec:background}.

\subsection{Skew Normal Distribution}
Normal distributions are symmetric and thus cannot model data skewness.
The skew normal distribution is developed to generalize normal distributions in this regard \citep{ohagan1978bayes}.
Since then, a vast number of extensions have been proposed; see \citet{arellano2006unification} for a few examples.
In particular, a sufficiently general version is called the unified skew normal (SUN) distribution,
a multivariate distribution that subsumes various existing versions as special cases \citep{arellano2006unification}.
The particular version that we use in \S\ref{sec:method} is the (univariate) extended skew normal distribution \citep[][Section 2.2]{azzalini2013skew}, which is also a special case of the SUN distribution.

The generality of the SUN distribution comes at a cost of complexity.
Its density function and cumulative distribution function both depend on the multivariate normal CDF, which in general is challenging to compute in high dimensions.
In addition, there is no direct method for sampling from the SUN distributions.
But fortunately SUN sampling can be reduced to sampling from linearly truncated multivariate normal distributions.
The latter has been studied extensively, and various specialized Monte Carlo methods exist \citep[\eg,][]{botev2017normal,gessner2020integrals,wu2024fast}.

Remarkably, the moment generating function of the SUN distribution is available in closed-form \citep[][]{arellano2006unification,arellano2022some}.
This property gives rise to the closed-form expression of the knowledge gradient in \S\ref{sec:method}.

\subsection{Skew Gaussian Processes}
Skew Gaussian processes generalize Gaussian processes by using the SUN distribution as the prior \citep{benavoli2020skew}.
Skew GPs contain standard GPs as special cases and better capture certain non-Gaussian data characteristics.
Indeed, skew GPs have been applied to various settings where standard GPs require approximate inference, particularly those involving preferential observations, cardinal data, and monotonicity constraints \citep[][]{benavoli2020skew,benavoli2021preferential,benavoli2024linearly,takeno2023towards}.
Our work is orthogonal to these modeling developments, and we leave developing the knowledge gradient acquisition function for skew GPs as future work.

\subsection{Preferential Bayesian Optimization}
The preferential likelihood \eqref{eq:preferential-probit-likelihood} is originally proposed by \citet{chu2005preference}.
Then, \citet{gonzalez2017preferential} develop the first acquisition function for preferential BO based on dueling bandits.
Some recent developments in this area include replacing the GP surrogate with an exact skew GP \citep[\eg,][]{takeno2023towards,benavoli2021preferential}, variants of the comparison oracle \citep{mikkola2020projective}, and regret bounds for preferential BO \citep{xu2024principled}.
The knowledge gradient for preferential BO has been considered by \citet{lin2022preference,astudillo2023qeubo}.
However, both of them employ approximation and thus the exact knowledge gradient remains unknown prior to our work.

%% file: paper/discussion.tex
We have developed a new knowledge gradient acquisition function for preference Bayesian optimization.
This acquisition function is exact and yet can be computed in a closed form despite that the look-ahead posterior is non-Gaussian.
Our formulation is based on two key observations.
First, the outcomes of pairwise comparisons are finite and thus the outer expectation in the knowledge gradient does not require Monte Carlo approximation.
Second, the one-step look-ahead posterior mean, a key quantity in the inner expectation of the knowledge gradient, surprisingly has a closed-form available in preference learning.

Various extensions are immediately possible.
Recall that the knowledge gradient belongs to a broader class of general look-ahead acquisition functions.
We have shown the one-step look-ahead posterior mean is available in a closed form.
More generally, all moments of the look-ahead posterior are available in closed forms \citep[\eg,][]{arellano2022some}.
Therefore, one can construct a broader class of look-ahead acquisition functions for preference learning beyond the knowledge gradient.
We hope that our work will lay out the basis for this direction.

%% file: appendix/extended-skew-normal.tex
In this section, we review some basic properties of extended skew normal distributions.
Everything in this section can be found in the textbook by \citet[][Section 2.2]{azzalini2013skew}.
\begin{definition}[Standard Form]
A univariate random variable $Z$ follows a standard extended skew normal distribution with parameters $\alpha \in \Rb$ and $\tau \in \Rb$ if and only if its density is of the form
\[
p(z; \alpha, \tau) = \frac{1}{\Phi(\tau)} \phi(z) \Phi\bb[\big]{\tau \sqrt{1 + \alpha^2} + \alpha z},
\]
where $\phi$ is the standard normal PDF and $\Phi$ is the standard normal CDF.
We denote it as $Z \sim \operatorname{ESN}(\alpha, \tau)$.
\end{definition}
General extended skew normal distributions are constructed from the standard extended skew normal distributions by shifting and scaling.
Let $Z \sim \operatorname{ESN}(\alpha, \tau)$ be a standard extended skew normal random variable.
Then, any random variable of the type $\xi + \omega Z$ is an extended skew normal random variable whose density is available through a change of variables.
\begin{definition}[General Form]
A univariate random variable $X$ follows an extended skew normal distribution with parameters $\xi \in \Rb$, $\omega > 0$, $\alpha \in \Rb$, and $\tau \in \Rb$ if and only if its density is of the form
\[
p(x; \xi, \omega^2, \alpha, \tau) =
\frac{1}{\omega \Phi(\tau)}
\phi\bb[\Big]{\frac{x - \xi}{\omega}}
\Phi\bb[\Big]{\tau \sqrt{1 + \alpha^2} + \alpha \frac{x - \xi}{\omega}},
\]
where $\phi$ is the standard normal PDF and $\Phi$ is the standard normal CDF.
We denote it as $X \sim \operatorname{ESN}(\xi, \omega^2, \alpha, \tau)$.
\end{definition}
The moment-generating function of an extended skew normal distribution $\operatorname{ESN}(\xi, \omega^2, \alpha, \tau)$ writes
\[
M(t) = \frac{1}{\Phi(\tau)} \exp\bb[\Big]{\xi t + \frac12 \omega^2 t^2} \Phi\bb[\Big]{\tau + \frac{\alpha \omega}{\sqrt{1 + \alpha^2}} t}.
\]
As a result, its cumulant-generating function writes
\[
K(t) = \log M(t) = -\log \Phi(\tau) + \xi t + \frac12 \omega^2 t^2 + \log \Phi\bb[\Big]{\tau + \frac{\alpha \omega}{\sqrt{1 + \alpha^2}} t}.
\]

%% file: appendix/proofs.tex
\begin{lemma}
\label{thm:conditional-mean-bivariate-normal}
Given a bivariate normal distribution
\[
\bb*{
\begin{matrix}
x_1 \\ x_2
\end{matrix}
}
\sim
\Nc\bb*{
\bb*{
\begin{matrix}
\mu_1 \\ \mu_2
\end{matrix}
},
\bb*{
\begin{matrix}
s_{11} & s_{12} \\
s_{21} & s_{22}
\end{matrix}
}
},
\]
the conditional expectation
\[
\Eb\sbb{x_1 \mid x_2 \geq 0}
=
\mu_1
+
\frac{
    \phi(\tau)
}{
    \Phi(\tau)
}
\cdot
\frac{
    s_{21}
}{
    \sqrt{s_{22}}
},
\]
where $\tau = \frac{\mu_2}{\sqrt{s_{22}}}$, $\phi$ is the normal PDF, and $\Phi$ is the normal CDF.
\end{lemma}
\begin{proof}
By Bayes' theorem, the conditional density writes
\[
p\bb{x_1 \mid x_2 \geq 0}
=
\frac{1}{
    \Phi\bb[\big]{\frac{\mu_2}{\sqrt{s_{22}}}}
} \cdot
\frac{1}{\sqrt{s_{11}}}
\phi\bb[\bigg]{
    \frac{x_1 - \mu_1}{\sqrt{s_{11}}}
} \cdot
\Phi\bb*{
    \frac{
        \mu_2 + s_{21} s_{11}\inv (x_1 - \mu_1)
    }{
        \sqrt{s_{22} - s_{21} s_{11}\inv s_{12}}
    }
},
\]
where $\Phi(\,\cdot\,)$ is the standard normal CDF and $\phi(\,\cdot\,)$ is the standard normal PDF.
Clearly, this is an extend skew normal (ESN) distribution with parameters
\[
\alpha = \frac{s_{21} s_{11}^{-\frac12}}{\sqrt{s_{22} - s_{21} s_{11}\inv s_{12}}}, \quad
\tau = \frac{\mu_2}{\sqrt{s_{22}}}, \quad
\xi = \mu_1, \quad
\omega = \sqrt{s_{11}}.
\]
See \S\ref{sec:extended-skew-normal} for background on extended skew normal (ESN) distributions.
The ESN cumulant-generating function is available in a closed-form; see \S\ref{sec:extended-skew-normal}.
Differentiating the cumulant-generating function $K(t)$ at $t = 0$ gives the mean
\[
\Eb[x_1 \mid x_2 \geq 0]
=
\sbb*{\frac{\diff}{\diff t} K(t)}_{t=0}
=
\xi + \frac{\phi(\tau)}{\Phi(\tau)} \cdot \frac{\alpha \omega}{\sqrt{1 + \alpha^2}}.
\]
\end{proof}

%% file: appendix/additional-experiments.tex
The test functinions used in the benchmarks are as follows:
\begin{itemize}
\item Quadratic function (2D): $f(x) = \frac12 (x_1^2 + x_2^2)$ defined on $[-1, 1]^2$.
\item Branin function (2D)
\item Hartmann function (6D)
\item Ackley function (6D)
\item Alpine1 function (7D)
\item Levy function (6D)
\end{itemize}
Except the quadratic function, all other functions are taken from BoTorch \citep{balandat2020botorch}.
All test functions are originally designed for minimization, and thus we negate them to convert them into maximization problems.